\documentclass[letterpaper]{article} % DO NOT CHANGE THIS
\usepackage{aaai2027} % DO NOT CHANGE THIS

\usepackage[hyphens]{url} % DO NOT CHANGE THIS
\usepackage{graphicx} % DO NOT CHANGE THIS
\usepackage{natbib} % DO NOT CHANGE THIS OR ADD OPTIONS
\usepackage{caption} % DO NOT CHANGE THIS OR ADD OPTIONS
\usepackage{amsmath}
\usepackage{amssymb}
\usepackage{amsthm}

\newcommand{\E}{\mathbb{E}}
\newcommand{\Prob}{\mathbb{P}}

\newcommand{\given}{\,\vert\,}

\newcommand{\bracks}[1]{\left[#1\right]}
\newcommand{\set}[1]{\left\{#1\right\}}
\newcommand{\abs}[1]{\left\lvert#1\right\rvert}

\newcommand{\indicator}[1]{\mathbf{1}\!\set{#1}}

\providecommand{\argmin}{\operatorname*{arg\,min}}

\newcommand{\Ent}{\mathrm{H}}
\newcommand{\MI}{\mathrm{I}}
\newcommand{\KL}{\mathrm{D}_{\mathrm{KL}}}
\newcommand{\IG}{\operatorname{IG}}

\newcommand{\World}{\Theta}
\newcommand{\Target}{Z}
\newcommand{\Data}{\mathcal{D}}
\newcommand{\History}{\mathcal{H}}
\newcommand{\CompSet}{\mathcal{C}}
\newcommand{\DecisionSet}{\mathcal{A}}

\newcommand{\decision}{a}
\newcommand{\metapolicy}{\mu}

\newcommand{\Loss}{L}
\newcommand{\Risk}{R}
\newcommand{\CompCost}{\tau}
\newcommand{\Budget}{B}
\newcommand{\VOC}{\operatorname{VOC}}

\newcommand{\Astar}{\textsc{A*}}

\newcommand{\GBFS}{\textsc{gbfs}}
\newcommand{\MCTS}{\textsc{mcts}}
\newcommand{\UCT}{\textsc{uct}}

\newcommand{\UCB}{\textsc{ucb}}

\newtheorem{definition}{Definition}
\newtheorem{theorem}{Theorem}
\newtheorem{proposition}{Proposition}

\newcommand{\secref}[1]{Section~\ref{#1}}

\newcommand{\eqnref}[1]{Equation~\eqref{#1}}
\newcommand{\defref}[1]{Definition~\ref{#1}}
\newcommand{\propref}[1]{Proposition~\ref{#1}}
\newcommand{\thmref}[1]{Theorem~\ref{#1}}

\title{Search as Computation Allocation}

\author{Alexander Tuisov}
\affiliations{Technion - Israel Institute of Technology}

\begin{document}

\maketitle

\begin{abstract}
Many algorithms spend an internal resource before returning a decision and
are evaluated only by the quality of that terminal output.  We formalize
such procedures as terminal computation-allocation problems: costly
computations produce observations, update beliefs about a latent
environment, and matter only through terminal decision loss.  Bellman
equations characterize optimal allocation under fixed budgets, priced
computation, and exact certification.  We then relate value of computation
(VOC) to information.  Mutual information equals myopic VOC under log loss,
whereas under simple regret VOC is a knowledge-gradient quantity; moreover,
information gain can rank computations arbitrarily poorly, although it gives
a one-sided upper bound on VOC.  Bandit pulls, tree simulations, and node
expansions illustrate the same model under different computation topologies.
Finally, under an explicit frontier-resolution and heuristic-error model,
maximizing approximate VOC recovers weighted A*, with A* and greedy
best-first search as limiting cases.  The theory identifies a shared
decision problem without asserting that one acquisition rule is universally
optimal.
\end{abstract}

% Optional anonymous links to code, data, or an extended version belong here,
% between the abstract and the main body.
%
% \begin{links}
%   \link{Code}{https://example.com/anonymous-code}
% \end{links}

\section{Introduction}
\label{sec:introduction}

Many algorithms spend an internal resource before returning a final
decision.  The resource may fund observations, simulations, expansions,
evaluations, or other internal operations.  When those computations have no
direct object-level utility, the central question is always the same:
\emph{which computation should be performed next to improve the decision
made when computation stops?}

This perspective matters because familiar exploration language can obscure
the actual objective.  A rollout does not earn object-level reward merely
because the bandit rule used inside \MCTS{} assigns it a reward.  A node
expansion is not valuable merely because it is novel or uncertain.
Information about an optimal decision can likewise be useless when it cannot
change that decision or when the associated value gap is negligible.
Computation is instrumentally valuable through its effect on terminal loss.

Rational metareasoning already treats internal computations as actions and
selects them by their expected value
\cite{russell1991right,russell1991principles,hay2012selecting}.  We specialize
that idea to \emph{terminally evaluated} procedures and use the specialization
to give a domain-independent model of internal resource allocation.  The
scope restriction is substantive.  It includes any finite procedure whose
internal actions acquire or process information, incur accounted costs, and
affect utility only through a terminal output.  It excludes settings in which
an intermediate action itself changes object-level utility or the external
world, such as cumulative-regret bandits or deliberation interleaved with
execution.

The unification occurs at the level of the metalevel decision problem, not
at the level of a universal heuristic.  Computations may be repeatable or
one-shot, may update one belief or many, and may reveal new available
computations.  These different observation kernels and computation
topologies lead to different practical policies even when the terminal
objective is shared.

\paragraph{Contributions.}
First, we define a terminal computation-allocation problem that separates the
latent environment, the decision-relevant target, the terminal loss, and the
available computations.  We give its explicit metalevel decision process
with a stop action, and Bellman equations characterize optimal allocation for
fixed-budget, cost-sensitive, and certificate-seeking objectives.

Second, we distinguish decision value from information.  We show that
conditional mutual information is exactly myopic value of computation under
logarithmic loss.  Under simple regret, the corresponding quantity is instead
the knowledge gradient.  We construct problems in which maximizing
information gain chooses a computation with an arbitrarily small fraction of
the best available VOC, and prove a complementary one-sided bound showing
that low information limits VOC for bounded target-dependent losses.

Third, we use pure-exploration bandits, \MCTS, and heuristic graph search as
contrasting instantiations, making explicit how their computation topologies
differ.  Finally, we give a concrete derivation of weighted \Astar.  Under a
frontier-resolution approximation and a common location model for heuristic
error, maximizing expected incumbent improvement orders nodes by $g+w h$.
The result recovers \Astar{} at $w=1$ and \GBFS{} as $w\rightarrow\infty$,
while clearly separating this bounded-budget derivation from the certificate
guarantees of exact \Astar.

The paper therefore advances a limited but testable thesis: terminally
evaluated search procedures are policies for the same kind of computation
allocation problem, and their familiar priority rules can be studied by
stating which approximations to terminal decision value they make.

\section{Preliminaries}
\label{sec:preliminaries}

% Target length: approximately 1--1.25 pages.
%
% Purpose:
% 1. Separate the latent environment from the decision-relevant target.
% 2. Define Bayes risk and simple regret.
% 3. Define target-directed mutual information.
%
% Do not define value of computation here. The next section should introduce
% computations, observation kernels, metalevel policies, and VOC as the
% paper's unifying framework.

\subsection{Bayesian Decision Problems}

\paragraph{Latent environment and terminal decision.}
Let $\World$ denote a latent environment drawn from a prior
$p(\theta)$.  After observing data $\Data$, an algorithm must return a
terminal decision $\decision\in\DecisionSet$ and incurs loss
$\Loss(\decision,\World)$.
The posterior-optimal terminal decision is
\begin{equation}
  \decision^\star(\Data)
  \in
  \argmin_{\decision\in\DecisionSet}
  \E\bracks{\Loss(\decision,\World)\given\Data},
  \label{eq:bayes-decision}
\end{equation}
and the corresponding Bayes risk is
\begin{equation}
  \Risk(\Data)
  =
  \min_{\decision\in\DecisionSet}
  \E\bracks{\Loss(\decision,\World)\given\Data}.
  \label{eq:bayes-risk}
\end{equation}

\paragraph{Simple regret.}
When $\World$ determines a value function $V_\World(\decision)$, the loss
associated with selecting $\decision$ can be taken to be simple regret,
\begin{equation}
  \Loss_{\mathrm{SR}}(\decision,\World)
  =
  \max_{\decision'\in\DecisionSet}V_\World(\decision')
  -
  V_\World(\decision).
  \label{eq:simple-regret}
\end{equation}
This loss measures only the quality of the final recommendation, rather than
the rewards or costs accumulated while computing it.

\paragraph{Decision-relevant target.}
Let
\begin{equation}
  \Target = \zeta(\World)
  \label{eq:decision-target}
\end{equation}
be the aspect of the environment relevant to the final decision: for example,
the identity of an optimizer, its value, or a feature sufficient for the
terminal choice.  We distinguish $\World$ from $\Target$ because the identity
of an optimizer alone generally does not determine the magnitude of regret.
Loss is therefore defined on $\World$, whereas information gain may target
the lower-dimensional variable $\Target$.

\subsection{Decision-Directed Information}

\paragraph{Conditional mutual information.}
For a prospective observation $Y$, its information about the target
$\Target$, conditioned on current data, is
\begin{equation}
  \MI(\Target;Y\given\Data)
  =
  \E_{Y\mid\Data}
  \bracks{
    \KL\!\left(
      p(\Target\mid\Data,Y)
      \,\middle\|\,
      p(\Target\mid\Data)
    \right)
  }.
  \label{eq:conditional-mi}
\end{equation}
For discrete $\Target$, this is equivalently the expected reduction in
posterior entropy,
\begin{equation}
  \MI(\Target;Y\given\Data)
  =
  \Ent(\Target\given\Data)
  -
  \E_{Y\mid\Data}
  \bracks{\Ent(\Target\given\Data,Y)}.
  \label{eq:entropy-reduction}
\end{equation}
The KL form in \eqnref{eq:conditional-mi} also applies when entropy itself is
not well behaved.

\paragraph{Information is target dependent.}
Information about the complete environment need not improve the terminal
decision.  We therefore measure information about $\Target=\zeta(\World)$
rather than indiscriminately about $\World$.  Even this targeted information
need not be proportional to regret reduction: resolving a low-probability
tie may reveal substantial information while having negligible decision
value.  This distinction motivates the value-of-computation objective
introduced in the next section.

\paragraph{Standing conventions.}
We assume that posterior expectations and conditional mutual information
exist, that ties among optimal decisions are resolved by a fixed rule, and
that the terminal decision set is finite unless stated otherwise.  The
Bayesian model represents uncertainty available to the resource-bounded
algorithm; it need not imply that the underlying problem is stochastic.

\section{Terminally Evaluated Computation Allocation}
\label{sec:computation-selection}

The objects introduced in \secref{sec:preliminaries} specify what an
algorithm must ultimately decide, but not how it may spend computation before
making that decision.  We now model this internal process.  The defining
restriction is that intermediate computations have no object-level value:
they matter only through their effect on the terminal decision and through
the resources they consume.
This metalevel view follows the rational-metareasoning tradition
\cite{russell1991right,hay2012selecting}; our purpose is to isolate the
terminally evaluated subclass independently of any particular algorithm
family.

\subsection{Computations and Histories}

At time $t$, the computation history is
\begin{equation}
  \History_t
  =
  \bigl((c_1,y_1),\ldots,(c_t,y_t)\bigr),
  \qquad
  \History_0=\varnothing.
  \label{eq:computation-history}
\end{equation}
Given a history $\History$, the algorithm may perform any computation
$c\in\CompSet(\History)$.  The computation has cost
$\CompCost(c,\History)>0$ and produces an observation $Y_c$ according to a
kernel
\begin{equation}
  p(y\mid c,\World,\History).
  \label{eq:observation-kernel}
\end{equation}
Both availability and cost may depend on the history.  Thus, expanding a
node may reveal previously unavailable computations, and two computations
need not have independent outcomes.  We write
$\History\oplus(c,y)$ for the history obtained by appending $(c,y)$.

A computation policy $\metapolicy$ maps each history to a distribution over
the available computations and a distinguished stop action.  On stopping,
a terminal rule $\delta$ returns
\begin{equation}
  \delta(\History)\in\DecisionSet.
\end{equation}
The policy may therefore adapt both what to compute and when to stop.

\begin{definition}[Terminal computation-allocation problem]
\label{def:tcap}
A terminal computation-allocation problem consists of a latent environment
$\World$, terminal decisions $\DecisionSet$, loss $\Loss$, available
computations $\CompSet(\History)$, observation kernels as in
\eqnref{eq:observation-kernel}, and computation costs $\CompCost$.  Internal
actions may change beliefs and which later computations are available, but
they affect evaluation only through the terminal loss and accumulated
computation cost.
\end{definition}

The last condition is the substantive one.  It includes any finite decision
procedure that spends an internal resource acquiring or processing
information before producing an output.  It excludes procedures in which an
internal action itself changes object-level utility or the external world.
For example, an observed payoff may serve only as information in pure
exploration, but it is itself utility in a cumulative-regret objective.

\subsection{Terminal Objectives}

Let
\begin{equation}
  K_t
  =
  \sum_{i=1}^{t}
  \CompCost(c_i,\History_{i-1})
  \label{eq:accumulated-computation-cost}
\end{equation}
denote accumulated computation cost.

\paragraph{Fixed budget.}
For budget $\Budget$, a policy must stop before its accumulated cost exceeds
$\Budget$.  Its objective is
\begin{equation}
  J_{\Budget}(\metapolicy,\delta)
  =
  \E\bracks{
    \Loss\bigl(\delta(\History_T),\World\bigr)
  },
  \qquad K_T\leq\Budget.
  \label{eq:fixed-budget-objective}
\end{equation}
This is the natural form for best-arm identification and for MCTS with a
fixed simulation budget.  When $\Loss$ is the loss in
\eqnref{eq:simple-regret}, \eqnref{eq:fixed-budget-objective} is expected
simple regret.  Thus, choosing this objective and taking computations to be
arm samples produces a pure-exploration bandit problem; taking them to be
tree simulations produces the usual terminally evaluated \MCTS{} problem.

\paragraph{Cost-sensitive stopping.}
When computation and decision quality share a common utility scale, a price
$\lambda>0$ yields
\begin{equation}
  J_{\lambda}(\metapolicy,\delta)
  =
  \E\bracks{
    \Loss\bigl(\delta(\History_T),\World\bigr)
    +
    \lambda K_T
  }.
  \label{eq:cost-sensitive-objective}
\end{equation}
The stopping time $T$ is selected by the policy.  This formulation trades
terminal quality directly against deliberation cost.

\paragraph{Exact or certificate-seeking search.}
In exact search, terminal loss is constrained rather than traded against
runtime:
\begin{equation}
  \begin{aligned}
    \min_{\metapolicy,\delta}\quad & \E[K_T] \\
    \text{subject to}\quad &
    \Loss\bigl(\delta(\History_T),\theta\bigr)=0,
    \quad \forall\theta\in\Omega(\History_T),
  \end{aligned}
  \label{eq:certificate-objective}
\end{equation}
where $\Omega(\History)$ is the set of environments consistent with a
history.
The terminal decision may include both a solution and a certificate of its
optimality.  This formulation captures algorithms such as A*, whose goal is
not to reduce nonzero simple regret at termination, but to reach a sound
zero-loss stopping condition using as little computation as possible
\cite{hart1968astar,dechter1985astar}.

\subsection{The Induced Metalevel Decision Process}

The history formulation induces a fully specified metalevel Markov decision
process.  Its state is the current transcript $\History$.  The available
actions are
\begin{equation}
  \CompSet(\History)\cup\{\mathsf{stop}\}.
  \label{eq:metalevel-actions}
\end{equation}
For a computation $c$, the posterior predictive transition law is
\begin{equation}
\begin{split}
  &\Prob\bigl(
    \History_{t+1}=\History\oplus(c,y)
    \mid \History_t=\History,c
  \bigr)\\
  &\qquad =
  \int p(y\mid c,\theta,\History)\,
  p(d\theta\mid\History).
  \label{eq:metalevel-transition}
\end{split}
\end{equation}
Thus, the next state records both the selected computation and its observed
outcome.  No conditional-independence assumption is needed: the entire
transcript is Markov even when no smaller belief statistic is available.

The stop action chooses the Bayes-optimal terminal decision and incurs
\begin{equation}
  \Risk(\History)
  =
  \min_{\decision\in\DecisionSet}
  \E\bracks{\Loss(\decision,\World)\given\History}.
  \label{eq:history-bayes-risk}
\end{equation}
Under priced computation, action $c$ incurs immediate cost
$\lambda\CompCost(c,\History)$ before the transition above.  Under a fixed
budget, the metalevel state is augmented to $(\History,b)$ and only
computations with $\CompCost(c,\History)\leq b$ are feasible.  Under exact
search, stopping is permitted only at certifying histories, and computation
cost is accumulated until such a state is reached.

\paragraph{Scope of the representation.}
Any finite procedure satisfying the terminal-evaluation condition in
\defref{def:tcap} can be written in this form: take its internal transcript
as $\History$, its internal operations as computations, and their
conditional outputs as observations.  This observation is intentionally not
a theorem-level contribution.  It identifies the boundary of the framework.
Different algorithms remain different metalevel decision processes because
they have different computation sets, transition kernels, losses, costs, and
approximations for choosing an action.

\subsection{Bellman Characterization}

At any history $\History$, the metalevel policy faces the same basic choice:
stop and act using the information already obtained, or perform one more
computation, observe its outcome, and continue from the updated history.
The Bellman equations below make this comparison explicit.  They describe an
ideal optimal allocation of computation, rather than an assumption that the
equations can be solved efficiently.

Let $V_b(\History)$ be the minimum expected terminal loss attainable with
remaining budget $b$.  It is the better of stopping now and choosing a
feasible computation whose possible outcomes are followed by optimal use of
the remaining budget:
\begin{equation}
\begin{split}
  V_b(\History)
  =
  \min\Bigl\{
    &\Risk(\History),\\
    &\inf_{\substack{c\in\CompSet(\History)\\
                     \CompCost(c,\History)\leq b}}
      \E\bracks{
        V_{b-\CompCost(c,\History)}
        \bigl(\History\oplus(c,Y_c)\bigr)
        \given\History
      }
  \Bigr\}.
  \label{eq:fixed-budget-bellman}
\end{split}
\end{equation}
The first term is the loss from stopping; the second is the expected loss
from computing and continuing.  Thus, computation is useful only through how
its observation improves later terminal decisions.

When computation is priced instead of constrained by a hard budget, define
the continuation cost of choosing $c$ first as
\begin{equation}
  \begin{aligned}
    Q_{\lambda}(c,\History)
    ={}& \lambda\CompCost(c,\History) \\
    &+
    \E\bracks{
      V_{\lambda}\bigl(\History\oplus(c,Y_c)\bigr)
      \given\History
    } .
  \end{aligned}
  \label{eq:cost-sensitive-continuation}
\end{equation}
The policy compares the best such continuation with stopping:
\begin{equation}
  V_{\lambda}(\History)
  =
  \min\left\{
    \Risk(\History),
    \inf_{c\in\CompSet(\History)} Q_{\lambda}(c,\History)
  \right\}.
  \label{eq:cost-sensitive-bellman}
\end{equation}
It continues exactly when some computation has lower total expected loss
than $\Risk(\History)$.

\begin{proposition}[Bellman optimality]
\label{prop:bellman-optimality}
For a finite computation horizon and finite computation and observation
spaces, \eqnref{eq:fixed-budget-bellman} and
\eqnref{eq:cost-sensitive-bellman} attain their minima and characterize an
optimal deterministic metalevel policy.
\end{proposition}

\begin{proof}
At the final metalevel state, stopping with a Bayes-optimal terminal decision
is optimal.  Backward induction then compares stopping with every feasible
first computation, conditioning on its observation and applying the optimal
continuation value.  Because the action and observation spaces are finite,
each minimum is attained; randomization cannot improve upon the minimizing
action.
\end{proof}

Exact search changes only the stopping condition.  The policy may stop once
some $\decision\in\DecisionSet$ has zero loss for every
$\theta\in\Omega(\History)$ still consistent with the observed history.
Such a history is \emph{certifying}.  If $V_{\mathrm{cert}}(\History)$ denotes
the minimum expected computation cost still needed to obtain a certificate,
then it is zero at certifying histories; at every other history,
\begin{equation}
  \begin{aligned}
    V_{\mathrm{cert}}(\History)
    =
    \inf_{c\in\CompSet(\History)}
    \Bigl[
      &\CompCost(c,\History) \\
      &+
      \E\bracks{
        V_{\mathrm{cert}}\bigl(\History\oplus(c,Y_c)\bigr)
        \given\History
      }
    \Bigr].
  \end{aligned}
  \label{eq:certificate-bellman}
\end{equation}
For A*, for example, $c$ is a node expansion and a certifying history contains
enough information to prove that the incumbent solution is optimal.  The
recursion therefore asks which expansion is expected to reach such a proof
with the least remaining work.

Thus, bounded-budget identification, cost-sensitive deliberation, and exact
search differ in their stopping conditions, but share the same underlying
choice: which computation should be performed next.  The following section
examines the value of that choice and when information gain is an adequate
surrogate.

\section{Value of Computation and Information}
\label{sec:value-information}

The Bellman equations in \secref{sec:computation-selection} characterize
optimal computation allocation, but solving them is generally as difficult as
solving the original search problem.  This section isolates the value
assigned to an individual computation, derives common one-step
approximations, and identifies the precise relationship between decision
value and information gain.

\subsection{Dynamic and Myopic Value of Computation}

Consider a history $\History$ at which the algorithm may either stop or
perform a computation.  Under a remaining budget $b$, the \emph{dynamic value
of computation} is
\begin{equation}
  \VOC_b^\star(c\mid\History)
  =
  \Risk(\History)
  -
  \E\bracks{
    V_{b-\CompCost(c,\History)}
    \bigl(\History\oplus(c,Y_c)\bigr)
    \given\History
  }.
  \label{eq:dynamic-voc-budget}
\end{equation}
It measures the reduction in terminal risk obtained by performing $c$ and
then allocating all remaining computation optimally.  For the
cost-sensitive objective, the corresponding net value is
\begin{equation}
\begin{split}
  \VOC_\lambda^\star(c\mid\History)
  =
  \Risk(\History)
  -
  \Bigl[
    &\lambda\CompCost(c,\History)\\
    &+
    \E\bracks{
      V_\lambda\bigl(\History\oplus(c,Y_c)\bigr)
      \given\History
    }
  \Bigr].
  \label{eq:dynamic-voc-cost}
\end{split}
\end{equation}
By \eqnref{eq:cost-sensitive-bellman}, it is optimal to stop exactly when
$\VOC_\lambda^\star(c\mid\History)\leq0$ for every available computation.

Dynamic VOC includes the option value of computations enabled by $c$ and of
all subsequent adaptive choices.  A common approximation assumes that the
algorithm stops immediately after the next observation
\cite{russell1991right,hay2012selecting}.  The resulting \emph{myopic VOC} is
\begin{equation}
  \VOC_1(c\mid\History)
  =
  \Risk(\History)
  -
  \E\bracks{
    \Risk\bigl(\History\oplus(c,Y_c)\bigr)
    \given\History
  }.
  \label{eq:myopic-voc}
\end{equation}
Its cost-sensitive form is
\begin{equation}
  \VOC_{1,\lambda}(c\mid\History)
  =
  \VOC_1(c\mid\History)
  -
  \lambda\CompCost(c,\History).
  \label{eq:net-myopic-voc}
\end{equation}

\begin{proposition}[Nonnegative decision value]
\label{prop:voc-nonnegative}
For every computation $c$ with a well-defined posterior,
$\VOC_1(c\mid\History)\geq0$.
\end{proposition}

\begin{proof}
After observing $Y_c$, the terminal rule can ignore the observation and use
the decision that was Bayes-optimal at $\History$.  Its expected loss is then
$\Risk(\History)$ by the tower property.  Reoptimizing the terminal decision
for each observation cannot increase that loss.
\end{proof}

Dividing VOC by computation cost can be a useful acquisition heuristic, but
it is not the Bellman-optimal treatment of cost in general.  The
cost-sensitive objective subtracts $\lambda\CompCost$ as in
\eqnref{eq:net-myopic-voc}; a hard budget compares continuation values as in
\eqnref{eq:dynamic-voc-budget}.  A value-to-cost ratio requires additional
assumptions about divisible resources or a particular relaxation of the
budget constraint.

\subsection{When Information Gain Is Decision Value}

For a decision-relevant target $\Target=\zeta(\World)$, define the
information gained from computation $c$ by
\begin{equation}
  \IG(c\mid\History)
  =
  \MI(\Target;Y_c\mid\History).
  \label{eq:computation-information-gain}
\end{equation}
Information gain is itself a value of computation for a particular terminal
decision problem.

\begin{theorem}[Log-loss equivalence]
\label{thm:log-loss-equivalence}
Suppose $\Target$ is discrete, the terminal decision is a predictive
distribution $q$ over $\Target$, and terminal loss is logarithmic,
\begin{equation}
  \Loss(q,\Target)=-\log q(\Target).
  \label{eq:log-loss}
\end{equation}
Then
\begin{equation}
  \VOC_1(c\mid\History)
  =
  \MI(\Target;Y_c\mid\History)
  =
  \IG(c\mid\History).
  \label{eq:voc-equals-information}
\end{equation}
\end{theorem}

\begin{proof}
The Bayes-optimal prediction is
$q(\cdot)=p(\Target=\cdot\mid\History)$, whose expected log loss is
$\Ent(\Target\mid\History)$.  Substituting this Bayes risk into
\eqnref{eq:myopic-voc} gives
\[
  \Ent(\Target\mid\History)
  -
  \E_{Y_c\mid\History}
  \bracks{\Ent(\Target\mid\History,Y_c)},
\]
which is conditional mutual information
\cite{cover2006elements}.
\end{proof}

Thus, maximizing mutual information is exactly myopic VOC when the final task
is probabilistic prediction under log loss.  For search and best-action
identification, however, the terminal loss is normally zero-one loss, simple
regret, solution cost, or certificate time.  Entropy then ceases to be the
Bayes risk, and the equivalence no longer holds.

\subsection{Zero-One Loss and Simple Regret}

\paragraph{Zero-one identification.}
Suppose the terminal decision is an estimate
$\widehat z\in\mathcal Z$ and
$\Loss(\widehat z,\Target)=\indicator{\widehat z\neq\Target}$.  Then
\begin{equation}
  \Risk(\History)
  =
  1-\max_{z\in\mathcal Z}p(z\mid\History),
  \label{eq:zero-one-risk}
\end{equation}
and consequently
\begin{equation}
\begin{split}
  \VOC_1(c\mid\History)
  =
  &\E\bracks{
    \max_{z\in\mathcal Z}
    p\bigl(z\mid\History,Y_c\bigr)
    \given\History
  }\\
  &-
  \max_{z\in\mathcal Z}p(z\mid\History).
  \label{eq:zero-one-voc}
\end{split}
\end{equation}
The relevant quantity is improvement in posterior confidence in the most
probable decision, not reduction in the entropy of the entire posterior.

\paragraph{Simple regret.}
Let $U_\World(a)$ be the value of terminal decision $a$ in environment
$\World$, and let
\begin{equation}
  m_a(\History)
  =
  \E\bracks{U_\World(a)\given\History}.
\end{equation}
With simple-regret loss,
\begin{equation}
  \Loss_{\mathrm{SR}}(a,\World)
  =
  \max_{a'\in\DecisionSet}U_\World(a')
  -
  U_\World(a).
  \label{eq:general-simple-regret}
\end{equation}

\begin{proposition}[Simple-regret VOC]
\label{prop:simple-regret-voc}
For the loss in \eqnref{eq:general-simple-regret},
\begin{equation}
\begin{split}
  \VOC_1(c\mid\History)
  =
  &\E\bracks{
    \max_{a\in\DecisionSet}
    m_a\bigl(\History\oplus(c,Y_c)\bigr)
    \given\History
  }\\
  &-
  \max_{a\in\DecisionSet}m_a(\History).
  \label{eq:knowledge-gradient-voc}
\end{split}
\end{equation}
This is the knowledge-gradient value of measuring $c$.
\end{proposition}

\begin{proof}
The Bayes risk is
\[
  \E\bracks{\max_a U_\World(a)\given\History}
  -
  \max_a m_a(\History).
\]
After observing $Y_c$, the conditional expectation of the first term changes,
but its expectation over $Y_c$ does not, by the tower property.  Subtracting
the expected posterior risk from the current risk leaves exactly
\eqnref{eq:knowledge-gradient-voc}.
\end{proof}

Equation~\eqref{eq:knowledge-gradient-voc} is the classical one-step
knowledge-gradient criterion \cite{frazier2008knowledge}.  It values an
observation according to how much it can improve the expected value of the
decision ultimately selected.  Mutual information instead values posterior
distinguishability, irrespective of the magnitude of the associated
decision loss.

\subsection{Information and Decision Value Can Diverge}

The divergence persists even when information is measured about the optimal
terminal decision itself.

\begin{proposition}[Arbitrarily poor information-gain ranking]
\label{prop:information-separation}
For every $\rho>0$, there is a finite simple-regret problem with two
equal-cost computations such that maximizing information gain selects a
computation whose myopic VOC is less than $\rho$ times the maximum available
myopic VOC.
\end{proposition}

\begin{proof}
Let $X\sim\operatorname{Bernoulli}(1/2)$ and
$W\sim\operatorname{Bernoulli}(q)$ be independent, where
$0<q<1/2$, and let the optimal terminal decision be
$\Target=(X,W)$.  A decision is a pair
$a=(\widehat X,\widehat W)$ with loss
\begin{equation}
  \Loss_\epsilon(a,\Target)
  =
  \frac{
    \epsilon\indicator{\widehat X\neq X}
    +
    \indicator{\widehat W\neq W}
  }{1+\epsilon}.
  \label{eq:separation-loss}
\end{equation}
This is simple regret for utilities
$U_\Target(a)=-\Loss_\epsilon(a,\Target)$.

Computation $c_X$ reveals $X$ and computation $c_W$ reveals $W$.  Measured in
bits,
\[
  \IG(c_X)=1,
  \qquad
  \IG(c_W)=h_2(q)<1,
\]
so information gain selects $c_X$.  However,
\[
  \VOC_1(c_X)
  =
  \frac{\epsilon}{2(1+\epsilon)},
  \qquad
  \VOC_1(c_W)
  =
  \frac{q}{1+\epsilon}.
\]
For $\epsilon<2q$, VOC selects $c_W$, and the ratio between the value selected
by information gain and the maximum available value is
$\epsilon/(2q)$.  Choosing $\epsilon<2q\rho$ proves the claim.
\end{proof}

The example separates probability from consequence.  The fair bit $X$
contains more uncertainty, but predicting it incorrectly matters almost
nothing.  The biased bit $W$ contains less information, but resolving it
prevents a much larger terminal loss.  An information objective cannot see
this asymmetry unless it is encoded in the target and scoring rule.

\subsection{A One-Sided Information Bound}

Although information does not determine decision value, sufficiently little
information limits how valuable an observation can be.

\begin{theorem}[Information upper bound]
\label{thm:information-upper-bound}
Suppose terminal loss depends on the environment only through $\Target$ and
satisfies $0\leq\Loss(a,\Target)\leq L_{\max}$.  If mutual information is
measured in nats, then
\begin{equation}
  \VOC_1(c\mid\History)
  \leq
  L_{\max}
  \sqrt{\frac{
    \MI(\Target;Y_c\mid\History)
  }{2}}.
  \label{eq:information-voc-bound}
\end{equation}
\end{theorem}

\begin{proof}
Let $p$ be the posterior of $\Target$ at $\History$ and $p_y$ its posterior
after observing $Y_c=y$.  The Bayes-risk functional is
$L_{\max}$-Lipschitz in total variation:
\[
  \abs{\Risk(p)-\Risk(p_y)}
  \leq
  L_{\max}\operatorname{TV}(p,p_y).
\]
Taking expectations, applying Pinsker's inequality, and then Jensen's
inequality gives
\[
\begin{split}
  \VOC_1(c\mid\History)
  &\leq
  L_{\max}
  \E_{Y_c\mid\History}
  \bracks{
    \sqrt{
      \frac{\KL(p_{Y_c}\,\|\,p)}{2}
    }
  }\\
  &\leq
  L_{\max}
  \sqrt{
    \frac{
      \E_{Y_c\mid\History}
      \bracks{\KL(p_{Y_c}\,\|\,p)}
    }{2}
  },
\end{split}
\]
and the expectation of the KL divergence is conditional mutual information.
\end{proof}

The bound is deliberately one-sided.  Low information implies low myopic
decision value for a bounded target-dependent loss, but high information
need not imply high decision value, as
\propref{prop:information-separation} shows.  Moreover, if simple regret
depends on value gaps not determined by the identity of the optimizer, then
the optimizer alone is not a sufficient target for
\thmref{thm:information-upper-bound}; one must enrich $\Target$ or retain the
full environment $\World$.

\subsection{Consequences for Search}

Three conclusions follow.  First, exact VOC is a property of a terminal loss
and an optimal continuation policy, not of uncertainty alone.  Second,
myopic VOC may be zero for a computation that enables valuable future
computations.  This failure occurs in tree search when one sample cannot
change the current best action \cite{hay2012selecting,tolpin2012mcts}, and
the same structural problem arises in certificate-seeking search when an
expansion is valuable only through the descendants it exposes.  Third,
mutual information is exact only for particular scoring rules and otherwise
supplies, at best, a surrogate or bound.

Accordingly, the unified view is not that terminally evaluated searches all
maximize information gain.  They all allocate computation to reduce terminal
decision loss.  Information-directed rules, confidence bounds, heuristic
priorities, and certificate-driven expansions are different approximations
to that common value-of-computation problem.

\section{Illustrative Instantiations}
\label{sec:instantiations}

The formalism above is intentionally agnostic about what a computation
does.  We use pure-exploration bandits, \MCTS, and heuristic graph search as
three contrasting examples, not as an exhaustive list of covered
procedures.  The terminal objective can be shared while the topology of the
available computations differs.  This difference is important when
replacing dynamic VOC by a one-step approximation.

\subsection{Pure-Exploration Bandits}

Consider a finite set of arms with latent mean rewards
$\World=(\mu_1,\ldots,\mu_K)$.  The terminal decision is a recommended arm,
and its simple-regret loss is
\begin{equation}
  \Loss(i,\World)
  =
  \max_j \mu_j-\mu_i.
  \label{eq:bandit-simple-regret}
\end{equation}
A computation $c_i$ samples arm $i$ and observes a reward from its posterior
predictive distribution.  Unless the budget has been exhausted, every arm
remains available after every sample.  Substituting arm sampling into
\propref{prop:simple-regret-voc} gives the knowledge-gradient criterion:
sample the arm whose observation is expected to increase the posterior value
of the arm ultimately recommended \cite{frazier2008knowledge}.

The resulting metalevel problem is \emph{flat}: computations change beliefs,
but normally do not reveal new computation types.  This is why repeated
one-step allocation rules can be effective in bandits, although they need
not be dynamically optimal.

\subsection{Monte Carlo Tree Search}

For \MCTS, let $\World$ encode the unknown values or unobserved outcomes in
the search tree.  The terminal decision is a root action, with loss equal to
its simple regret.  A computation may be a rollout, expansion, or leaf
evaluation.  Its observation updates beliefs about every ancestor on its
path and may also make computations at newly generated descendants
available.

Thus, \MCTS{} combines a bandit-like terminal decision at the root with a
hierarchical computation topology.  A locally uncertain leaf has value only
insofar as resolving it can change the root recommendation.  Conversely, a
computation can have zero myopic VOC while enabling a descendant with
positive dynamic VOC.  This is the distinction between static and dynamic
values of computation emphasized in metareasoning accounts of Monte Carlo
search \cite{hay2012selecting,sezener2020mcts}.  \UCT{} uses count-based
confidence bonuses to allocate simulations
\cite{kocsis2006uct,tolpin2012mcts}.  In our framework, these bonuses are
acquisition approximations rather than consequences of uncertainty having
intrinsic terminal value.  \secref{sec:uct-derivation} makes this claim
precise by deriving the exact one-step root value and then identifying the
approximations that produce the \UCT{} score.

\subsection{Heuristic Graph Search}

Let a search node $n$ represent a path from an initial state to a state
$s(n)$.  Write $g(n)$ for the known cost of that path and
\begin{equation}
  C^\star(n)
  =
  \text{minimum additional cost from }s(n)\text{ to a goal},
  \label{eq:frontier-cost-to-go}
\end{equation}
with $C^\star(n)=+\infty$ for a dead end.  The best solution represented by
the node has total cost
\begin{equation}
  X_n=g(n)+C^\star(n).
  \label{eq:frontier-solution-cost}
\end{equation}
Although the underlying graph may be deterministic, $C^\star(n)$ is
epistemically uncertain before the relevant portion of the graph has been
searched.  A posterior model can condition on the generated graph, heuristic
values, learned features, and all previous expansions.

At a history $\History$, the available computations are expansions or
evaluations of frontier nodes.  Expanding $n$ produces an observation $Y_n$
containing its successors, edge costs, goal or dead-end evidence, and any
new heuristic values.  Unlike a bandit sample, this computation is usually
performed once and changes the future computation set by exposing
descendants.

Suppose the search has an incumbent solution of cost $U(\History)$.  Under a
fixed budget, the terminal rule returns the best solution found by the end
of the run.  Its loss may be written as solution-cost regret,
\begin{equation}
  \Loss(\pi,\World)=C(\pi)-C^\star,
  \label{eq:solution-cost-regret}
\end{equation}
with a finite failure penalty when no incumbent exists.  The value of
expanding $n$ is therefore determined by the probability that information
below $n$ improves the incumbent, the size of that improvement, and any
later computations the expansion enables.

Exact search uses a different terminal condition.  If $h(n)$ is admissible,
\begin{equation}
  \ell(n)=g(n)+h(n)
  \label{eq:astar-lower-bound}
\end{equation}
is a lower bound on the cost of every solution through $n$.  A history is
certifying once
\begin{equation}
  U(\History)
  \leq
  \min_{n\in\mathrm{OPEN}(\History)}\ell(n).
  \label{eq:astar-certificate}
\end{equation}
The familiar \Astar{} ordering expands a node of minimum $\ell(n)$ and is
tied to reaching this exact certificate
\cite{hart1968astar,dechter1985astar}.  Under bounded computation, however,
the aim is instead to minimize terminal solution regret.  The next section
shows how a weighted rather than certifying ordering emerges from a simple
approximation to that objective.

\section{UCT as Recursive Optimistic Allocation}
\label{sec:uct-derivation}

The \MCTS{} instantiation admits an exact one-step computation value before
any confidence-bound approximation is introduced.  Let
$\mathcal A_0$ be the root actions and let
\begin{equation}
  m_a(\History)
  =
  \E\bracks{Q_\World(s_0,a)\given\History}
  \label{eq:mcts-root-posterior-mean}
\end{equation}
be the posterior mean value of root action $a$.  A computation $c$ is a
simulation, expansion, or leaf evaluation with observation $Y_c$.  Applying
\propref{prop:simple-regret-voc} at the root gives
\begin{equation}
\begin{split}
  \VOC^{\mathrm{root}}_1(c\mid\History)
  ={}&
  \E\bracks{
    \max_{a\in\mathcal A_0}
    m_a\bigl(\History\oplus(c,Y_c)\bigr)
    \given\History
  }\\
  &-
  \max_{a\in\mathcal A_0}m_a(\History).
  \label{eq:mcts-root-voc}
\end{split}
\end{equation}
This is an exact myopic rule for terminal root simple regret.  It values a
simulation only through the change it can induce in the final root choice.
Dynamic VOC replaces the posterior risk after one simulation by the optimal
continuation value under the remaining simulation budget.

For intuition, suppose a simulation updates only the posterior mean of root
action $a$.  Let
\begin{equation}
  M_{-a}(\History)
  =
  \max_{b\in\mathcal A_0\setminus\{a\}}m_b(\History)
\end{equation}
and write $m'_a=m_a(\History\oplus(c,Y_c))$.  Then
\begin{equation}
\begin{split}
  \VOC^{\mathrm{root}}_1(c\mid\History)
  ={}&
  \E\bracks{\max\{m'_a,M_{-a}\}\given\History}\\
  &-\max\{m_a,M_{-a}\}.
  \label{eq:mcts-two-alternative-voc}
\end{split}
\end{equation}
The score depends on the posterior tail that can cross the current decision
boundary and on the value gap across that boundary.  A visitation count
alone is not sufficient.

\subsection{From Root Decision Value to a Local Bound}

\UCT{} can nevertheless be obtained from
\eqnref{eq:mcts-root-voc} through two explicit approximations.  First,
\emph{localization} replaces the effect of a simulation on the root decision
by its effect on the identity of the best child at the current node.  The
same local selection problem is then solved recursively along a path.
Second, \emph{optimism} replaces expected improvement across the local
decision boundary by a high-confidence upper bound on each child value.

To see the resulting score, suppose returns observed after selecting action
$a$ at node $s$ lie in $[0,1]$ and are conditionally independent with fixed
mean $Q(s,a)$.  Let $N(s)$ be the number of visits to $s$,
$N(s,a)$ the number selecting $a$, and $\widehat Q(s,a)$ the corresponding
sample mean.  Hoeffding's inequality gives
\begin{equation}
  \Prob\left(
    Q(s,a)>
    \widehat Q(s,a)+
    \sqrt{\frac{\log(1/\delta)}{2N(s,a)}}
  \right)
  \leq\delta.
  \label{eq:uct-hoeffding}
\end{equation}
Choosing a confidence schedule
$\delta_s=N(s)^{-\kappa}$ yields the optimistic value
\begin{equation}
  U_\kappa(s,a)
  =
  \widehat Q(s,a)
  +
  \sqrt{
    \frac{\kappa\log N(s)}
         {2N(s,a)}
  }.
  \label{eq:uct-confidence-score}
\end{equation}
Selecting a child maximizing $U_\kappa(s,a)$ at every node and recursively
descending until a leaf is reached is exactly the \UCT{} tree policy, up to
the conventional exploration constant
\cite{auer2002finite,kocsis2006uct}.

The complete approximation chain is therefore
\begin{equation}
\begin{split}
  \text{dynamic root VOC}
  &\;\longrightarrow\;
  \text{myopic root VOC}\\
  &\;\longrightarrow\;
  \text{local decision value}\\
  &\;\longrightarrow\;
  \text{upper confidence bound}\\
  &\;\longrightarrow\;
  \UCT.
  \label{eq:uct-approximation-chain}
\end{split}
\end{equation}

\subsection{Interpretation}

This derivation explains both the strength and the mismatch of \UCT.  The
confidence bonus is a tractable proxy for the possibility that an
undersampled action is better than its current estimate.  Recursive
localization makes the proxy inexpensive and naturally routes simulations
through the tree.  But it discards the root value gaps, the probability that
a local change propagates to the root, correlations among backed-up values,
and the option value of computations exposed below a leaf.

Moreover, the usual UCB motivation concerns cumulative regret, whereas the
\MCTS{} decision is commonly evaluated by terminal simple regret
\cite{tolpin2012mcts}.  Our derivation does not claim that \UCT{} exactly
maximizes root VOC.  It shows instead how \UCT{} results when a
decision-directed acquisition objective is localized and replaced by a
uniform confidence heuristic.  Equation~\eqref{eq:mcts-root-voc} supplies
the stronger target against which \UCT{} and alternative tree policies can
be compared.

\section{Weighted A* as Approximate Value of Computation}
\label{sec:weighted-astar}

The exact dynamic VOC of a node expansion accounts for all descendants it
may expose and every later adaptive choice.  Computing it would generally
require solving the search-control problem itself.  We therefore study a
one-step \emph{frontier-resolution approximation}: computation $c_n$ is
treated as if it revealed the best solution cost $X_n$ available through
frontier node $n$.  This is exact for a macro-computation that solves the
subproblem below $n$, and an approximation when $c_n$ is a single expansion.

\subsection{Expected Incumbent Improvement}

Assume first that an incumbent of cost $U=U(\History)$ is available.  After
resolving $n$, the terminal rule returns the cheaper of the incumbent and
the revealed solution through $n$.  For the loss in
\eqnref{eq:solution-cost-regret}, the reduction in loss is
\begin{equation}
  U-\min\{U,X_n\}=(U-X_n)_+,
  \label{eq:incumbent-improvement}
\end{equation}
where $(x)_+=\max\{x,0\}$.  The approximate myopic VOC is consequently
\begin{equation}
  \widehat{\VOC}_1(c_n\mid\History)
  =
  \E\bracks{(U-X_n)_+\given\History}.
  \label{eq:frontier-expected-improvement}
\end{equation}
This is the minimization analogue of expected improvement: a node is useful
only when it can produce a solution better than the incumbent, and its value
also reflects the possible magnitude of that improvement.

Equation~\eqref{eq:frontier-expected-improvement} is not an information-gain
criterion.  Learning with certainty that a branch is mediocre may be
informative, but it has no one-step value once that branch cannot improve
the terminal solution.

\subsection{A Location Model for Heuristic Error}

We now impose a deliberately simple posterior model.  Conditional on the
current history, suppose the total solution cost through every frontier node
has the form
\begin{equation}
  X_n
  =
  g(n)+w h(n)+\varepsilon_n,
  \qquad w>0,
  \label{eq:weighted-location-model}
\end{equation}
where the marginal distribution of $\varepsilon_n$ is the same for all
frontier nodes.  Independence among the errors is not required.  The
coefficient $w$ calibrates the heuristic to the posterior location of
remaining solution cost; common residual errors express the approximation
that nodes with the same $g+w h$ are equally promising.

\begin{theorem}[Weighted best-first ordering]
\label{thm:weighted-astar-voc}
Suppose \eqnref{eq:frontier-expected-improvement} and
\eqnref{eq:weighted-location-model} hold and all computations have equal
cost.  Selecting a frontier node of minimum
\begin{equation}
  f_w(n)=g(n)+w h(n)
  \label{eq:weighted-astar-score}
\end{equation}
maximizes approximate myopic VOC.  If expected improvement is strictly
decreasing at the attained frontier scores, the two orderings coincide up
to ties.
\end{theorem}

\begin{proof}
Let $\varepsilon$ have the common marginal distribution of the
$\varepsilon_n$, and define
\begin{equation}
  \phi_U(x)
  =
  \E\bracks{(U-x-\varepsilon)_+}.
  \label{eq:expected-improvement-location}
\end{equation}
For every fixed value of $\varepsilon$, $(U-x-\varepsilon)_+$ is
nonincreasing in $x$; hence $\phi_U$ is nonincreasing.  By
\eqnref{eq:weighted-location-model},
\[
  \widehat{\VOC}_1(c_n\mid\History)
  =
  \phi_U\bigl(g(n)+w h(n)\bigr).
\]
Every minimizer of $f_w$ is therefore a maximizer of approximate VOC.
Strict decrease gives the converse, modulo equal scores.
\end{proof}

The result derives a family of acquisition rules rather than selecting one
universal weight.  At $w=1$, the score is $g(n)+h(n)$ and gives \Astar{}
ordering.  Finite $w>1$ gives weighted \Astar{} ordering
\cite{pohl1970heuristic}.  Finally, dividing the score by $w$ shows that
\begin{equation}
  \frac{f_w(n)}{w}
  =
  h(n)+\frac{g(n)}{w}
  \longrightarrow h(n),
  \label{eq:gbfs-limit}
\end{equation}
so \GBFS{} is the limiting ordering as $w\rightarrow\infty$.  This limit is
appropriate only when variation in cost-to-come is negligible relative to
heuristic progress, or when a satisficing terminal loss makes solution
discovery dominate solution cost.

\subsection{What the Derivation Does and Does Not Establish}

\paragraph{Posterior calibration.}
In \eqnref{eq:weighted-location-model}, $w$ is the relative calibration of
the heuristic and the exact path cost.  If
$C^\star(n)=h(n)+\varepsilon_n$, then $w=1$ and the VOC proxy recovers
\Astar.  If the heuristic systematically understates variation in remaining
cost, a larger coefficient is appropriate.  More generally, the model
suggests a history-dependent weight $w(\History)$ estimated from heuristic
error, rather than a fixed constant chosen independently of the posterior.

\paragraph{Terminal objective.}
The theorem concerns solution quality under a bounded budget, after an
incumbent has been found.  It does not derive the certification behavior of
exact \Astar.  Under admissibility, $g+h$ is both a posterior score in the
special location model and the deterministic lower bound used in
\eqnref{eq:astar-certificate}.  These are distinct reasons for the same
ordering.  For $w>1$, $g+w h$ is generally not a valid lower bound, so the
VOC derivation should not be read as an optimality certificate.

\paragraph{Approximation boundary.}
A single expansion rarely reveals $X_n$.  It instead produces a partial
observation and enables computations at its children.  Exact dynamic VOC
would propagate the value of those contingent descendants through
\eqnref{eq:fixed-budget-bellman}.  The frontier-resolution model collapses
that continuation value into a common residual distribution.  Likewise, if
expansion costs differ, the cost-sensitive one-step score becomes
\begin{equation}
  \phi_U\bigl(f_w(n)\bigr)
  -
  \lambda\CompCost(c_n,\History),
  \label{eq:cost-sensitive-weighted-score}
\end{equation}
which need not preserve weighted-\Astar{} ordering.

These qualifications are the point of the derivation.  Weighted \Astar{} is
not asserted to be the exact solution to every metalevel search problem.
Rather, it is the exact myopic policy produced by a transparent set of
approximations: incumbent-improvement loss, frontier resolution, a common
location family for heuristic error, and equal computation cost.  Relaxing
these assumptions yields principled departures from the classical score,
while retaining the same terminal value-of-computation objective.

\section{Related Work}
\label{sec:related-work}

\paragraph{Value of information and rational metareasoning.}
Statistical decision theory values an observation by the improvement in the
decision made after observing it \cite{howard1966information,berger1985decision}.
Rational metareasoning applies the same principle to internal computation:
computations become metalevel actions, and a bounded agent selects among them
using value of computation
\cite{russell1991right,russell1991principles}.  Subsequent work develops
optimal and approximate computation-selection policies and emphasizes the
difference between myopic value and the value of computations enabled later
\cite{hay2012selecting,tolpin2012semimyopic}.  Our framework is a
specialization of this tradition, not an alternative to it.  Its role is to
isolate terminally evaluated computation allocation, state its
domain-independent metalevel decision process, and characterize exactly when
information gain agrees with decision value.

\paragraph{Bandits and information-directed allocation.}
Classical \UCB{} methods optimize cumulative reward and hence attach utility
to actions taken throughout the run \cite{auer2002finite}.  Pure-exploration
bandits instead optimize a terminal recommendation
\cite{bubeck2011pure,russo2020bestarm}, placing them inside our scope.
Knowledge-gradient policies select measurements by expected improvement in
the final decision \cite{frazier2008knowledge}.  Information-theoretic
analyses of Thompson sampling and information-directed sampling connect
learning about an optimal action to regret
\cite{russo2016information,russo2018ids}.  Our results complement these
methods by separating mutual information from myopic simple-regret VOC:
the two coincide for log loss, but not for general terminal decision losses.

\paragraph{Monte Carlo tree search.}
\UCT{} imports a cumulative-regret bandit rule into tree search
\cite{kocsis2006uct,browne2012survey}.  Because an internal simulation is
normally valuable through the final root action, simple-regret and
value-of-information alternatives have been proposed
\cite{tolpin2012mcts}.  Static and dynamic values of computation make the
effect of future simulations explicit \cite{sezener2020mcts}.  We retain this
metareasoning interpretation while treating \MCTS{} as one instance of a
broader terminal computation-allocation problem.  The comparison with
bandits and graph search highlights a structural distinction: tree
computations both update existing beliefs and expose contingent
computations below the frontier.

\paragraph{Heuristic search.}
\Astar{} and generalized best-first search are classically analyzed through
lower bounds, admissibility, and node-expansion guarantees
\cite{hart1968astar,dechter1985astar}.  Early weighted heuristic search
studies how the relative reliance on path cost and estimated remaining cost
affects search efficiency \cite{pohl1970heuristic}.  Related metareasoning
work values whether an additional heuristic should be computed for a node
\cite{tolpin2013rationalastar}.  Our weighted-\Astar{} result asks a
different question: under which posterior and one-step approximations does
the value of expanding a frontier node itself become monotone in $g+w h$?
This interpretation supplements rather than replaces the worst-case and
certificate-based analyses of heuristic search.

\section{Limitations}
\label{sec:limitations}

The framework is normative: applying it requires a prior or predictive model
for unperformed computations.  In deterministic search this uncertainty is
epistemic, and a misspecified model can misrank computations.  Exact dynamic
VOC is itself generally intractable because it includes all adaptive
continuations.  Myopic VOC avoids that recursion but can assign zero value to
an expansion whose descendants enable useful later computations.

The information results are also asymmetric.  Mutual information is exact
only for particular scoring rules such as log loss.  For bounded
target-dependent loss, our bound shows that little information implies
little myopic VOC, but high information need not imply high decision value.
If value gaps are not determined by the chosen target, the target must be
enriched or the full environment retained.

The weighted-\Astar{} derivation has additional assumptions: an incumbent is
already available, resolving a frontier node is treated as revealing its
best completion, heuristic errors share a common location family, and
computation costs are equal.  These assumptions yield an exact ordering
within the approximation, not completeness, optimality, or bounded
suboptimality guarantees for arbitrary graph search.  Unequal expansion
costs, correlated observations with different marginals, dead-end
probabilities, or the option value of descendants can all produce acquisition
rules that are not expressible as $g+w h$.

Finally, the paper does not address cumulative-regret interaction, real-time
search in which execution changes the world during deliberation, or empirical
estimation of the proposed posterior quantities.  These cases require either
a broader object-level utility model or additional learning assumptions.

\section{Conclusion}
\label{sec:conclusion}

Search algorithms can be compared at a level beneath their familiar
priority formulas: each is a policy for choosing costly computations before
a terminal decision.  The appropriate normative quantity is reduction in
terminal decision loss, not information in isolation.  This distinction
explains both the usefulness and the limits of information-directed search.

The framework recovers bandit measurement, Monte Carlo simulation, and node
expansion as instances with different observation and availability
structures.  Its weighted-\Astar{} derivation further illustrates how a
classical priority rule can arise from explicit approximations to VOC, with
\Astar{} and \GBFS{} marking interpretable endpoints.  A natural next step is
to replace fixed acquisition formulas by learned, history-dependent
approximations of dynamic VOC while preserving the terminal objective that
the search is actually evaluated on.

\bibliography{references}

% Uncomment only if AAAI-27 requires the checklist inside the paper.
% \input{ReproducibilityChecklist}

\end{document}